\documentclass[preprint,12pt]{elsarticle}

\usepackage[margin=1in]{geometry}
\usepackage{dsfont}
\usepackage{bm}
\usepackage{amssymb}
\usepackage{amsthm}
\usepackage{amsmath}
\usepackage{xcolor}
\newtheorem{lemma}{Lemma}
\usepackage{bookmark}
\usepackage{makecell, comment}
\usepackage{color}

\usepackage[ruled,vlined,linesnumbered]{algorithm2e}
\SetKwInput{kwParameters}{Parameters}
\SetKwInput{kwHyperparameters}{Hyperparameters}
\usepackage[flushleft]{threeparttable}
\usepackage{graphicx}
\usepackage{booktabs} 
\usepackage{multirow}
\usepackage{longtable}
\usepackage{tabularx}
\usepackage[figuresright]{rotating}
\usepackage{colortbl}
\theoremstyle{definition}

\definecolor{myblue}{RGB}{0,112,192}

\journal{Engineering}

\begin{document}

\begin{frontmatter}




\title{ 
Toward Cooperative Driving in Mixed Traffic: An Adaptive Potential Game-Based Approach with Field Test Verification
}

\address[label1]{Key Laboratory of Road and Traffic Engineering of the State, Ministry of Education, Tongji University, Shanghai, 201804, China}
\address[label2]{State Key Laboratory of Intelligent Green Vehicle and Mobility, Tsinghua University, Beijing, 100084, China}
\address[label3]{State Key Lab of Intelligent Transportation System, School of Transportation Science and Engineering, Beihang University, Beijing, 100191, China }

\author[label1]{Shiyu Fang}

\author[label1]{Xiaocong Zhao}

\author[label1]{Xuekai Liu}

\author[label1]{Peng Hang \corref{cor1}}  

\author[label2]{Jianqiang Wang} 

\author[label3]{\\ Yunpeng Wang}

\author[label1]{Jian Sun  \corref{cor1}}

\cortext[cor1]{Corresponding author, email: \href{hangpeng@tongji.edu.cn}{hangpeng@tongji.edu.cn} (P. Hang), \href{sunjian@tongji.edu.cn}{sunjian@tongji.edu.cn} (J. Sun)}

\begin{abstract}
Connected autonomous vehicles (CAVs), which represent a significant advancement in autonomous driving technology, have the potential to greatly increase traffic safety and efficiency through cooperative decision-making. However, existing methods often overlook the individual needs and heterogeneity of cooperative participants, making it difficult to transfer them to environments where they coexist with human-driven vehicles (HDVs).
To address this challenge, this paper proposes an adaptive potential game (APG) cooperative driving framework. First, the system utility function is established on the basis of a general form of individual utility and its monotonic relationship, allowing for the simultaneous optimization of both individual and system objectives. Second, the Shapley value is introduced to compute each vehicle's marginal utility within the system, allowing its varying impact to be quantified. Finally, the HDV preference estimation is dynamically refined by continuously comparing the observed HDV behavior with the APG's estimated actions, leading to improvements in overall system safety and efficiency. 
Ablation studies demonstrate that adaptively updating Shapley values and HDV preference estimation significantly improve cooperation success rates in mixed traffic. Comparative experiments further highlight the APG's advantages in terms of safety and efficiency over other cooperative methods. Moreover, the applicability of the approach to real-world scenarios was validated through field tests.
\end{abstract}



\begin{keyword}
Connected autonomous vehicles \sep Cooperative driving \sep Potential game \sep Adaptive weight \sep Field test
\end{keyword}

\end{frontmatter}



\section{Introduction}

With the continuous advancement of autonomous driving and vehicle-to-everything (V2X) technologies, 
CAVs 
have emerged as a pivotal direction in the evolution of autonomous driving systems, with increasing attention given to how cooperative decision-making can enhance decision-making efficiency and safety. In response, the National Innovation Center of Intelligent and Connected Vehicles of China (CICV) has released a strategic roadmap for cooperative driving, designating cooperative driving as a key priority for the next stage of development. Additionally, the Society of Automotive Engineers (SAE) J3216 standard defines four levels of cooperative driving automation (CDA): state sharing, intent sharing, agreement seeking, and prescriptive. On the basis of this framework, the Federal Highway Administration (FHWA) initiated the CARMA program, demonstrating that cooperative decision-making can significantly enhance the performance of CAVs in complex scenarios. These efforts underscore the critical role of CAVs in advancing next-generation mobility systems.

Despite notable progress, existing cooperative decision-making frameworks remain limited in their ability to handle hybrid traffic, primarily because of three unresolved issues:
1) Idealized modeling of HDVs: Most studies simplify HDVs as deterministic or passive entities using car-following or binary cooperation--competition models \citep{cui2025vehicle, zhou2022Reasoning, chi2021game}.
2) Neglect of behavioral heterogeneity: Such simplifications ignore the diversity and adaptability of human driving preferences, which are crucial for realistic interaction modeling.
3) Insufficient real-world validation: Field experiments indicate that system efficiency can decrease by up to 14\% when they are applied to mixed-traffic conditions \citep{10026667}, revealing a substantial gap between simulation-based frameworks that focus on maximizing system efficiency and their real-world applicability.


Two fundamental limitations hinder current CAV frameworks in hybrid traffic. \textbf{System--Individual Decision Conflict}: While cooperative systems prioritize global objectives such as traffic throughput, individual vehicles balance safety, comfort, and efficiency, leading to inherent conflicts between system-wide and individual decision-making perspectives. Studies have shown that rigid system-centric optimization, which overlooks individual needs, can lead to a 5\% increase in the number of accidents in mixed-autonomy environments \citep{10123387}, as HDVs resist sacrificing their preferences for collective gains. \textbf{Cooperation Participants Heterogeneity}: Cooperative models often assume that HDVs are homogeneous, overlooking the behavioral diversity driven by individual decision preferences and the unsymmetrical impact on the system caused by different conflict relationships \citep{wang2024homogeneous, huang2021dynamic}. As a result, when faced with real-world diverse driving styles and complex conflict topologies, cooperative strategies may become suboptimal. Recent studies have shown that the introduction of CAVs does not always lead to improved traffic performance; under certain mixed-traffic conditions, CAV deployment may inadvertently exacerbate congestion or reduce network efficiency \citep{talebpour2016influence, wu2022influence}. In summary, the most pressing challenges in cooperative driving lie in optimizing system-level performance while ensuring individual equilibrium and in real-time quantification and adaptation of cooperative strategies on the basis of the diverse characteristics of participants \citep{huang2024adaptive, zheng2025oscillation, liu2025morning}.


\begin{figure*}[h]
  \begin{center}
  \centerline{\includegraphics[width=6.5in]{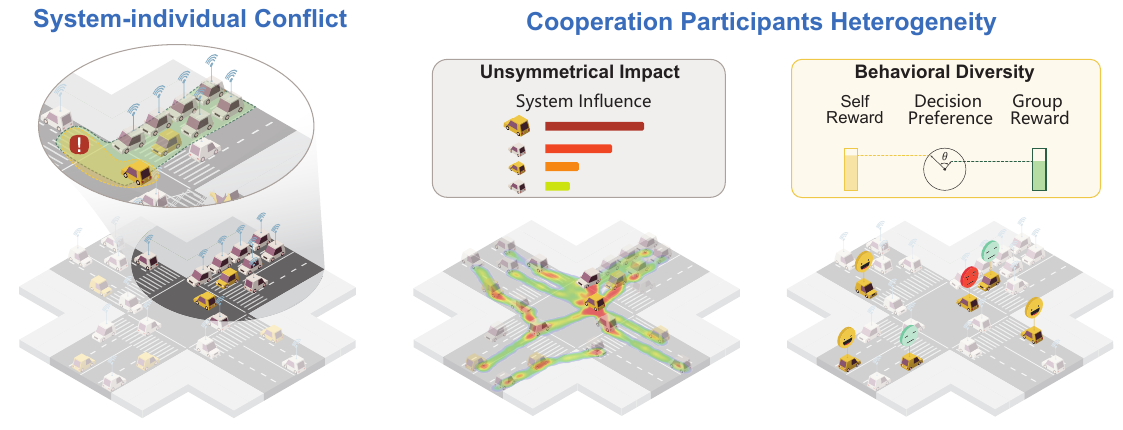}}
  \caption{Challenges of cooperative driving under mixed-traffic conditions. A) Conflicts between individual and system-wide objectives. B) Heterogeneity in participant impact and decision preferences. }\label{fig: thumbnail}
  \end{center}
  \vspace{-0.8cm}
\end{figure*}

To address the aforementioned issues, this paper proposes an adaptive potential game (APG) framework for cooperative decision-making in mixed-traffic environments. By establishing a monotonic relationship between individual and system utilities, system-level optimization also accounts for the needs of individual cooperative participants. Furthermore, an adaptive weight method is introduced to dynamically quantify each participant's impact on the system and update estimates of their decision preferences, enabling the framework to effectively handle heterogeneous participants within the cooperative system. The contributions of this study are summarized as follows:


\begin{enumerate}
\item Driving cooperation is formulated as a potential game-theoretic problem, where the system utility function is constructed in a bottom-up manner from the general expression of individual utility, ensuring an equipotential relationship between changes in individual and system utilities. With this design, achieving optimal system utility becomes a sufficient condition for reaching individual equilibrium among cooperative participants, thereby enabling the simultaneous optimization of both individual and system objectives.


\item The different impacts of cooperative participants and diverse HDV preferences are represented by two sets of weights. Shapley values are used to calculate each vehicle's marginal utility and its impact on the system, and a backward propagation method updates the HDV preferences by comparing estimated and observed behaviors. By adaptively updating these weights, the method identifies the characteristics of heterogeneous participants, thereby improving the adaptability of the model in mixed-traffic environments.




\item The experimental results demonstrate that the proposed APG framework outperforms existing methods across all penetration rates in terms of safety and efficiency. Additionally, a field test is conducted to demonstrate the applicability of the proposed method in real-world mixed-traffic environments.
\end{enumerate}

The remainder of the paper is organized as follows. Section \uppercase\expandafter{2} presents a literature review of the relevant research. Section \uppercase\expandafter{3} formulates the cooperative driving problem and introduces the cooperative driving framework for mixed traffic. Section \uppercase\expandafter{4} discusses the details of the APG framework. In Section \uppercase\expandafter{5}, we validate our cooperative driving framework through simulation experiments and field tests. Finally, conclusions are presented in Section \uppercase\expandafter{6}.

\section{Literature Review}
Achieving efficient and reliable cooperative decision-making has become an increasingly popular topic in recent years. Numerous methods have been proposed to address this challenge. Upon review, the prevailing approaches for cooperative driving can be categorized into optimization-based, rule-based, learning-based, and game-based methods.

\subsection{Optimization-based approaches}
The optimization-based approach aims to maximize or minimize the system utility function to achieve a specific goal. Owing to the intricate conflict relationships and multiple interacting opponents in mixed-traffic environments, studies often introduce simplifications to strike a balance between computation time and performance. A common practice is to partition the conflict area into multiple zones, each of which can be occupied by only a single vehicle at any given time \citep{yu2019managing, xu2019cooperative}. Another strategy is to decompose the cooperative driving problem into two layers. The upper layer generates a set of candidate trajectories on the basis of space--time conflict constraints, whereas the lower layer selects feasible trajectories from among the candidate trajectories to optimize an objective function \citep{zhang2021trajectory, hu2019trajectory}. However, as mentioned earlier, these methods focus primarily on system-wide optimization while overlooking individual needs, limiting their applicability in more realistic scenarios.

\subsection{Rule-based approaches}
Unlike the optimization-based method, which attempts to obtain a globally optimal solution, the rule-based method attempts to quickly generate a feasible solution through preset regulations. To achieve this goal, virtual rotation projection was used to map vehicles from multiple entrances onto a one-dimensional lane for simplification \citep{chen2021connected, hu2021constraint}. Building on this, an improved depth-first search tree (iDFST) was combined with a predefined communication topology to determine vehicle trajectories \citep{xu2018distributed}. However, reservation-based methods may not be suitable for environments with HDVs because of their uncontrollable behavior, which can raise safety issues for the cooperation system \citet{chen2022conflict}.

\subsection{Learning-based approaches}
Learning-based methods derive driving policies by either extracting features from expert drivers in real-world data or enabling agents to learn through continuous interaction with a simulated environment \citep{klimke2023automatic, shi2021connected}. \citet{ZHOU2024104807} proposed a reasoning graph-based reinforcement learning (RGRL) framework to manage conflicting platoons with complex interactions, integrating a graph convolution network (GCN) to capture multiagent cooperation features. \citet{fang2025towards} leveraged large language models and memory modules to enhance CAV performance through negotiation and cooperation. However, these methods rely primarily on data-driven patterns and numerical relationships and overlook the underlying logic of human decision-making.

\subsection{Game-based approaches}
Game theory formulates the relationship between incentive structures and players' strategies and has been widely used to replicate social decision-making \citep{schwarting2019social, camerer2006does}. Among the various methodologies, cooperative games (CGame) and coalition games are among the most common model-based approaches for cooperative decision-making, in which optimal strategies are derived through a superadditive payoff function \citep{zimmermann2018carrot}. These methods have been widely applied in various scenarios, including merging zones \citep{li2024cooperative}, roundabouts \citep{chandra2022gameplan, hang2021decision}, and intersections \citep{hang2021cooperative}. In addition, level-k game theory offers another viable approach by modeling human reasoning at different depths, generating decisions aligned with different cognitive levels \citep{liu2024cooperative}. Stackelberg games, in contrast, establish leader--follower relationships and construct decision trees over the action space to optimize cooperative strategies \citep{kang2024metaverses}. While game-theoretic models effectively capture human behavior in conflict scenarios, many existing approaches impose overly rigid assumptions about decision-making and often assume that participants are either purely cooperative or purely competitive. However, real-world HDVs exhibit diverse behavioral traits, necessitating adaptive preference modeling on the basis of observed behaviors. 

In summary, despite extensive efforts in cooperative decision-making across various scenarios, the lack of consideration for cooperation participant heterogeneity or the reliance on overly simplified assumptions has limited these methods to simulation-based validation. Achieving real-world cooperation remains a challenge.

\section{Problem Formulation and Proposed Framework}
\subsection{Problem formulation}
Cooperative driving has the potential to significantly improve the efficiency of traffic systems. However, in mixed-traffic environments, a key challenge lies in balancing individual vehicle needs with overall system objectives, making it essential to design cooperative strategies that account for both individual and collective utilities. It seeks to find a balance between system-wide utility and individual utility. Therefore, this naturally formulates a game-theoretic problem, namely,
\begin{equation}
    \mathbf{a}^* = \arg \max_{a_i \in A} \mathbb{E} \left[ \lambda \cdot \mathcal{S}(\mathbf{a}) + (1 - \lambda)\cdot\sum_{i\in N} u_i(a_i, \tilde{a}_{-i}) \right]
    \label{eq: problem formulation}
\end{equation}
where $\mathbf{a}^*$ is the optimal action list for the cooperative driving participants, $A$ and $N$ represent the action space and cooperation participant set, respectively, and $\lambda \in [0,1]$ is a tuning factor used to balance the weight between individual utility $u_i$ and system utility $\mathcal{S}$. $a_i$ represents the action of individual $i$. $\tilde{a}_{-i}$ refers to the actions of all participants except for individual $i$, where there is uncertainty for individual $i$.

Additionally, to focus on cooperative decision-making in mixed traffic, this study adopts the following assumptions: 1) In both simulations and field tests, when a vehicle enters a circular area with a radius \(R \) centered at the bottleneck, its position, speed, and other relevant information can be accurately obtained through vehicle-to-vehicle communication without bias. 2) Lane-changing behavior upstream of the bottleneck is ignored, and once a vehicle enters the cooperative system, it does not abruptly change its destination.

\subsection{Cooperative driving framework}
Despite the impressive performance of cooperative driving in fully CAV environments, its effectiveness sharply declines when it transitions to more realistic, human-involved scenarios. This is primarily due to insufficient consideration of individual needs and the heterogeneity of cooperation participants.

\begin{figure*}[t]
  \begin{center}
  \centerline{\includegraphics[width=6.5in]{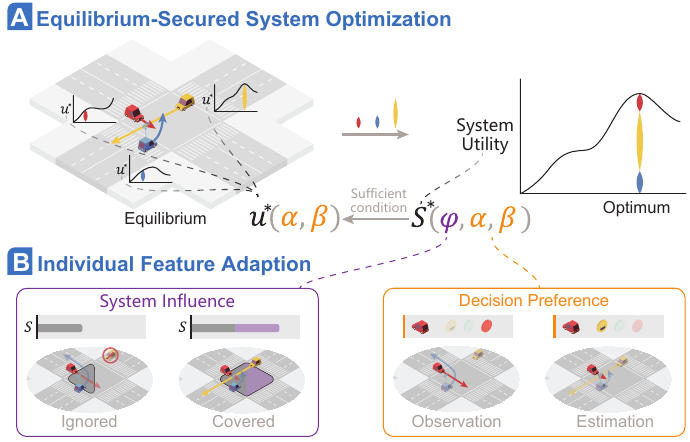}}
  \caption{Overview of the proposed APG cooperative driving framework. A) Equilibrium-secured system optimization: The system utility $\mathcal{S}$ is constructed from the individual utility $u$ in a bottom-up manner such that the optimality of system utility becomes a sufficient condition for individual equilibrium. B) Individual feature adaptation: Each individual's marginal contribution to the system utility is computed to quantify its influence $\varphi$. Afterward, by comparing the observed action of an individual with the APG-estimated action, the estimation of their decision preferences $\alpha$ and $\beta$ is updated.}\label{fig: pg-framework}
  \end{center}
  \vspace{-0.8cm}
\end{figure*}

To address this issue, this paper proposes an APG cooperative decision-making framework, as illustrated in Fig.~\ref{fig: pg-framework}. Constructing a system utility function that accommodates diverse individual needs is crucial for enhancing the applicability of cooperative methods in mixed-traffic environments. To this end, this paper starts from a general formulation of individual utility, and by establishing an equal potential relationship between individual and system utilities, the system utility function is derived from the bottom up. With this design, when the system utility reaches its optimum, no cooperative participant can unilaterally alter their behavior to gain additional benefit, thereby achieving a game-theoretic equilibrium among individuals. The simultaneous optimization of individual and system utilities is achieved.

Furthermore, this paper decomposes individual heterogeneity into two key aspects: the unsymmetrical impact on the system and the diversity of behavioral preferences. These characteristics are represented by two distinct sets of weights. The Shapley value is used to quantify each vehicle's contribution by calculating its marginal utility, reflecting its relative impact on the system. In addition, the decision preference estimation for HDVs is continuously refined through a back-propagation approach that compares the optimal cooperative actions generated by the APG with the observed behaviors of HDVs. This alignment gradually improves the consistency between the theoretical optimal action and observed actions, thereby increasing the safety and efficiency of the cooperative system.




\section{Methodology}
\subsection{Equilibrium-Secured System Optimization}
To optimize both system and individual objectives simultaneously, we first establish a fundamental formulation for individual driving decision utility. Decision-making is generally influenced by both individual strategies and the actions of others, resulting in a utility that encompasses both self-related and group-related factors. The self-related part refers to those utilities that are affected only by personal actions, whereas the group-related part is affected by both personal and other actions. By orchestrating two kinds of rewards, realistic selfish and altruistic driving behaviors can be reproduced. In addition, considering that frequent changes in action over a short period of time can cause a significant reduction in comfort, decisions are made with the foresight of several steps ahead. Inspired by \citet{10101707}, the individual utility is as follows:
\begin{equation}
    u_{i} = \sum_{t=0}^{T}\gamma^{t}(r_{s,i}^{t}(a_{i}) +  \sum\limits_{j\in N,j\neq i}r_{g,ij}^{t}(a_{i},a_{j})) 
    \label{eq: individual utility}
\end{equation}
where $r_{s,i}^{t}$ and $r_{g,ij}^{t}$ are the self-related reward and group-related reward, respectively; $T$ is the planning horizon; and $\gamma$ is the discount factor and represents the relative weight of immediate rewards versus future rewards, reflecting their diminishing importance in decision-making over time.

More specifically, $r_{s,i}^{t}$ and $r_{g,ij}^{t}$ are designed as follows:
\begin{equation}
\begin{aligned}
    & r_{s,i}^{t}(a_{i}) = a_{i}^{t} - d_{o,i}^{t} - j_{i}^{2,t} \\
    & r_{g,ij}^{t}(a_{i},a_{j}) = \Delta TTCP_{ij}^{t} = |\frac{d_{cp,i}^{t}}{v_{i}^{t}} - \frac{d_{cp,j}^{t}}{v_{j}^{t}}|
    \label{reward self}
\end{aligned}
\end{equation} 
where $d_{o,i}^{t}$ represents the distance to the desired destination, indicating the efficiency of the driving process; $j_{i}^{2,t}$ is the square of the acceleration, reflecting the comfort of the driving process; and $TTCP$ denotes the time to the collision point. The safety of the driving process is represented by calculating the current distance to the conflict point $d_c$ and the speed $v$ of the two vehicles.

However, modeling the utility of a cooperative system is more challenging because it involves multiple participants and complex relationships. Additionally, individual decision-making is often driven by personal preferences, whereas system-level decision-making aims to maximize overall efficiency while ensuring safety. This disparity gives rise to a fundamental conflict between user equilibrium and system optimality. Existing cooperative methods typically adopt a top-down approach and focus on maximizing the efficiency of the overall system. Nevertheless, these methods often neglect individual objectives, which limits their practicality in real-world applications. In addition, these methods overlook the mutual feedback between the system and individuals, where individual actions are influenced by other participants in the system, and the system state is the result of the aggregation of individual actions. 

Therefore, to establish an objective function that is practically feasible in real-world mixed scenarios, it is crucial to introduce constraints that simultaneously consider system-level and individual-level objectives. This demand naturally frames the cooperative driving problem of mixed traffic as a game. However, current research in game theory primarily addresses these dynamics through the lens of cooperation versus noncooperation. This binary perspective is an overly strong assumption that does not accurately reflect the diverse behaviors exhibited by real-world human drivers.

To address this issue, instead of predefining whether drivers act cooperatively or noncooperatively, this paper establishes a monotonic function that links system utility to individual utilities, formulated as follows:

\begin{equation}
    u_{i}(a_{i},a_{-i}) - u_{i}(x_{i},a_{-i}) > 0 \iff \mathcal{S}(a_{i},a_{-i}) - \mathcal{S}(x_{i},a_{-i}) > 0, \forall a_{i},x_{i} \in A
    \label{eq: potential function definition}
\end{equation}
where $x_{i}$ represents any participant's action within the action space $A$. The above equation holds for any participant in the cooperative system. Since the system utility and individual utility change in a way that preserves the same potential, this game is referred to as a potential game, and the system utility that satisfies Eq.~\ref{eq: potential function definition} is called the potential function.

On the basis of the fundamental definition of the potential function, the following lemma can be derived: 1) the existence of pure-strategy Nash equilibrium, 2) the finite improvement property, and 3) the global optimality of equilibrium. 

First, people have different tendencies for different strategies considering the given information. This results in a probability distribution of each strategy, known as the mixed strategy. Mathematically, the pure strategy is a special case of a mixed strategy, where the probability of one strategy equals one and those of the other strategies equal zero. Therefore, with the pure strategy, decision-makers choose a deterministic strategy. Lemma 1 states the existence of pure strategy NE if the strategy space is finite.

\begin{lemma}
	A finite potential game possesses at least one pure strategy NE.
\end{lemma}
\begin{proof}
	Let $P$ be the potential function for game $G(u_{1}, u_{2},...,u_{i})$, where $\boldsymbol{a}$ is a vector of all player strategies that satisfies
    \begin{center}
        $P(\boldsymbol{a}) \geq P(x_{i}, a_{-i}), \forall x_{i} \in A $ \\
    \end{center}
    Consequently, if $P$ reaches the maximum, then the potential game $G$ possesses a pure strategy NE, with $\boldsymbol{a}$ being the equilibrium point. In addition, if $P$ reaches its maximum at more than one point, potential game $G$ possesses multiple NEs.
\end{proof}

Moreover, improvement path $\boldsymbol{\eta}=(\boldsymbol{a}^{0},\boldsymbol{a}^{1},...,\boldsymbol{a}^{k})$ is an iterative path from the initial point to the NE, where for every $k\geq 1$ there exists a player that can improve the value of the potential function by adjusting his strategy. $\boldsymbol{a}^{k}=(x_{i}, \boldsymbol{a}_{-i}^{k-1})$ is the strategy vector at the $k$ step, which is influenced by player $i$'s strategy $x_{i}$ and satisfies $P(\boldsymbol{a}^{k}) \geq P(\boldsymbol{a}^{k-1})$. Moreover, $\boldsymbol{a}^{0}$ is called the initial point and represents every player's initial strategy, and the last element of improvement path $\boldsymbol{\eta}$ is called the terminal point, which is also an NE for the potential game $G$. 

\begin{lemma}
	Every potential game has a finite improvement property.
\end{lemma}
\begin{proof}
	For every improvement path $\boldsymbol{\eta}=(\boldsymbol{a}^{0},\boldsymbol{a}^{1},...,\boldsymbol{a}^{k})$, we have Eq.~\ref{eq: potential function definition}: 
    \begin{center}
    $P(\boldsymbol{a}^{0}) < P(\boldsymbol{a}^{1}) < P(\boldsymbol{a}^{2}) < ... < P(\boldsymbol{a}^{k})$ \\
    \end{center}
    As $A$ is a finite strategy space, the sequence $\boldsymbol{\eta}$ must be finite, and game $G$ has a finite improvement property.
\end{proof}

In our problem, owing to the dynamic constraints of the vehicles, their strategies lie within a finite space, thereby satisfying the requirements of a potential game.

Finally, the existence of multiple NEs may result in inconsistency within a cooperative system. In most game models, when multiple equilibrium points coexist, optimization is often achieved by methods such as risk-dominant equilibrium or Pareto equilibrium to select the optimal equilibrium. However, the heterogeneity of players may cause them to converge to different equilibrium points, potentially leading to system collapse. One of the key reasons for this issue is that most game models rely on strong assumptions about player behavior. 

In contrast, in potential games, instead of explicitly assuming that players are inherently cooperative or competitive, the model simply establishes the relationship between utilities to represent the interactions among players. Therefore, the potential game transforms the relation between multiple equilibrium points into the relation between the local optimal and global optimal and guarantees the uniqueness of the global optimal, which effectively prevents deadlock or collision. On the basis of Lemma 1, the global optimality of equilibrium can be deduced. Furthermore, this equilibrium point is the terminal point of the improvement path $\eta$.

\begin{lemma}
	A finite potential game has an equilibrium point that optimizes the potential function globally.
\end{lemma}
\begin{proof}
	Let $\boldsymbol{a}^{*}$ be the terminal point of improvement path $\boldsymbol{\eta}=(\boldsymbol{a}^{0},\boldsymbol{a}^{1},...,\boldsymbol{a}^{*})$. For any $\boldsymbol{a}^{i} \in \boldsymbol{\eta}$, there is
    \begin{center}
    $P(\boldsymbol{a}^{*}) \geq P(\boldsymbol{a}^{i})$ \\
    \end{center}
    According to the definition of a potential game, for every player $i$, $u_{i}(\boldsymbol{a}^{*}) \geq u_{i}(\boldsymbol{a}^{i})$, which leads to Nash equilibrium. Therefore, the terminal point $\boldsymbol{a}^{*}$ is the Nash equilibrium point, which is also the global optimum.
\end{proof}

This property ensures that the optimization of system utility naturally leads to the improvement of individual utilities, effectively reconciling the conflict between user equilibrium and system optimality. Moreover, since the potential function serves as a bridge between individual and system-level objectives, its maximization guarantees that no participant can further increase their own utility through unilateral deviation. As a result, an NE is established between the system and individuals, fostering a stable and efficient cooperative driving strategy in mixed-traffic environments.

Furthermore, according to the definition of the potential function and the design of individual utility, we can derive the expression for the system utility as follows:
\begin{equation}
    \mathcal{S}(\boldsymbol{a}) = \sum_{t=0}^{T}\gamma^{t}(\sum\limits_{i \in N}r_{s,i}^{t}(a_{i}) +  \sum\limits_{i \in N}\sum\limits_{j\in N,j< i}r_{g,ij}^{t}(a_{i},a_{j}))
    \label{eq: potential function}
\end{equation}

When the system utility is defined in the form of Eq.~\ref{eq: potential function}, Eq.~\ref{eq: potential function definition} is satisfied, and for any cooperative participant, the change in their individual utility equals the change in the system utility. In this case, the game is called the exact potential game. Therefore, solving for the optimal cooperative strategy is equivalent to solving for the optimal solution to Eq.~\ref{eq: potential function}.

However, owing to differences in the positions, states, and conflict relationships of the vehicles within the system, the impact of each vehicle on the system varies. During the optimization process, as system managers, we are more concerned with potential failure points that could lead to severe conflicts, deadlocks, or collisions rather than vehicles that have not yet entered or have already left the bottleneck area. To achieve this, the Shapley value is introduced to quantify these unsymmetrical impacts dynamically.

\subsection{Individual Feature Adaption}
In mixed-traffic scenarios, the uncertainty of cooperation is driven largely by individual heterogeneity, which stems from both the participant's impact on the system and their decision preferences. To address this issue, we quantify each vehicle's impact on the system using Shapley values, while a backward propagation method is employed to continuously update the estimations of individual preferences.

\subsubsection{Quantifying the unsymmetrical impact through the Shapley value}  %
The Shapley value was originally developed by \citet{shapley1953value} for n-player games in coalitional form with side payments, with the goal of determining each individual's share of the overall payoff on the basis of their contribution, following the principle of fair distribution. By calculating the expected marginal contribution, a unique weight for each player, which reflects their contribution to the system, can be derived through Eq.~\ref{eq: shapley value definition}. 
\begin{equation}
    \varphi_{i}(\mathcal{S}) = \sum_{n \subset N, i \in n} \frac{(N-|n|-1)!|n|!}{N!}[\mathcal{S}(n) - \mathcal{S}(n-\{i\})] \\
    \label{eq: shapley value definition}
\end{equation}
where $\varphi_{i}(\mathcal{S})$ is the Shapley value of player $i$ in game $\mathcal{S}$, $n$ is a coalition that is a subset of all player sets $N$, and $\mathcal{S}(n)$ and $\mathcal{S}(n-\{i\})$ are the contributions of coalition $c$ with and without player $i$, respectively. Therefore, their difference constitutes the marginal contribution of player $i$.

The formula for calculating the Shapley value is established by designing a function that satisfies several axioms \citep{winter2002shapley}. The linearity, symmetry, and null player property of the Shapley value ensure the effective operation of the cooperative system.

The linearity of the Shapley value indicates that for any two games, if the utilities of the two games are linearly combined, then each participant's Shapley value will be the linear combination of their respective Shapley values, which is as follows:
\begin{equation}
    \varphi (\alpha u_i+\beta u_j) = \alpha \varphi (u_i) + \beta \varphi (u_j)
    \label{eq: linearity}
\end{equation}
where $u_i$ and $u_j$ are two game problems and $\alpha$ and $\beta$ are arbitrary constants. The linearity of the Shapley value allows multiple game problems to be transformed into a single combined problem for the solution. Therefore, when individual utility is split into a self-related component and a group-related component and then linearly combined to obtain the system utility, the Shapley value remains unchanged. That is,
\begin{equation}
    \varphi_{i}(\mathcal{S}) = \varphi (u_i) =\varphi (r_{s,i}) + \sum\limits_{j\in N,j\neq i}\varphi(r_{g,ij})
    \label{eq: potential linearity}
\end{equation}

On the basis of Eq.~\ref{eq: potential linearity}, when the marginal utility of a vehicle is calculated to determine the Shapley value, only the individual utility of that vehicle must be computed. There is no need to calculate the marginal utility for each different coalition, which greatly simplifies the calculation.

Symmetry is another important property of the Shapley value, as shown in Eq.~\ref{eq: symmetry}.
\begin{equation}
    \varphi_{i}(\mathcal{S}) = \varphi_{j}(\mathcal{S})  \iff \mathcal{S}(c-\{i\}) = \mathcal{S}(c-\{j\}),  \forall n \subset N
    \label{eq: symmetry}
\end{equation}

The symmetry of the Shapley value ensures that its value depends solely on the vehicle's marginal utility rather than its inherent properties. This design prevents the current CAV decision-making process from primarily reacting to the behavior of HDVs instead of fostering equal interaction, thereby ensuring that the system is operated fairly.

Finally, the null player property of the Shapley value refers to the condition where if a participant's marginal utility is zero, their Shapley value is also zero. This participant no longer has any effect on the system.
\begin{equation}
    \mathcal{S}(c-\{i\}) = \mathcal{S}(c-\{j\}) \iff \varphi_{i}(\mathcal{S})=0
    \label{eq: null player property}
\end{equation}

Given the open boundaries of the traffic system, which allows for real-time updates of system participants, the null player property effectively addresses scenarios where system participants change. As vehicles gradually enter or leave bottleneck areas, the Shapley value can be adjusted accordingly to accurately measure each vehicle's impact on the system. More specifically, for an intersection, vehicles at the intersection may have a higher Shapely value because they are in a more dangerous area, resulting in an increase in the group-related reward and marginal contribution. In contrast, a vehicle leaving the system will have a Shapley value that gradually decreases to zero as its impact on the system diminishes. 

Considering that this study is oriented toward real-world engineering applications, we introduce a batch-based coarsening mechanism to improve the efficiency of Shapley value computation when computational resources are limited. This mechanism aggregates spatially or behaviorally correlated vehicles into a single supernode, thereby reducing the combinatorial complexity while preserving the key interaction relationships among participants. Each supernode $b$ represents a coalition of vehicles that exhibit highly similar dynamic characteristics within a short time horizon. The Shapley value of the supernode is first computed as $\varphi_b(\mathcal{S})$ and then redistributed to individual vehicles within the batch according to their singleton potentials, as shown in Eq.~\ref{eq:supernode_shapley}.

\begin{equation}
\varphi_{i}(\mathcal{S}) =
\begin{cases}
    \displaystyle
    \varphi_{b}(\mathcal{S}) \cdot 
    \frac{\mathcal{S}(\{i\})}
    {\sum_{j \in b} \mathcal{S}(\{j\})}, 
    & \text{if } \sum_{j \in b} \mathcal{S}(\{j\}) > 0, \\[10pt]
    \displaystyle
    \frac{\varphi_{b}(\mathcal{S})}{|b|}, 
    & \text{otherwise.}
\end{cases}
\label{eq:supernode_shapley}
\end{equation}
This hierarchical approximation inhibits the combinatorial explosion, enabling near real-time computation while preserving the fairness and interpretability of Shapley-based credit assignment. The approximation error is bounded and can be adjusted by tuning the batch size threshold $|b|$, which controls the granularity of vehicle aggregation.

After the Shapley value is introduced to the game, the original exact potential game is extended to a weighted potential game, where the change in individual utility is proportional to the change in system utility. By incorporating the Shapley value, the framework effectively identifies key agents that have a greater impact on the cooperative system, ensuring that their behaviors receive higher priority during the optimization process. This allows the system to allocate resources and adjust decision-making strategies more efficiently, ultimately enhancing overall performance.

Furthermore, considering the various decision preferences of real-world human drivers, their observed behavior may significantly deviate from the optimal cooperative strategy. To address this issue, we continuously refine HDV preference estimations by comparing observed actions with estimated ones.

\subsubsection{Adjusted estimation of HDVs with back-propagation} 
After the Shapley value is introduced, the cooperative problem is modeled as a weighted potential game. By optimizing the system utility that satisfies the definition of a potential function, the optimal cooperative actions can be obtained, ensuring equilibrium between each individual and the system. However, in the real world, many uncontrolled HDVs may not obey these optimal cooperative actions. Instead, they tend to make decisions on the basis of their own preferences and decision logic. Additionally, some self-interested HDVs may prioritize their own needs over the collective good, which can lead to deadlock or collisions. Therefore, CAVs must maintain an accurate understanding of the HDVs in the system and adjust the optimal cooperative decisions accordingly.

Considering that behavior is fundamentally driven by intrinsic preferences, aligning optimal cooperative actions with observed HDV behavior can be approximated by accurately identifying human driver preferences. Therefore, a set of relative weights is designed to represent the driving preferences of different HDVs. In this process, the actions generated by the APG with respect to the HDVs are referred to as the estimated action. By comparing the differences between the estimated action and the observed action, we update the estimated preference weights to help the system generate the optimal cooperative action that better aligns with the observed action of HDVs.

By incorporating the relative weights that represent the individual driving preferences of the drivers into Eq.~\ref{eq: individual utility}, we obtain the following:
\begin{equation}
    u_{i} = \sum_{t=0}^{T}\gamma^{t}(\alpha_{i}\cdot r_{s,i}^{t}(a_{i}) +  \sum\limits_{j\in N,j\neq i}\beta_{i}\cdot r_{g,ij}^{t}(a_{i},a_{j})) 
    \label{eq: individual utility preference}
\end{equation}
where $\alpha$ and $\beta$ are two real numbers that depict the vehicle's preference when it makes decisions. A larger value of $\alpha$ indicates a greater focus on self-related benefits, making the individual more likely to exhibit aggressive behavior. In contrast, a larger value of $\beta$ signifies a greater sensitivity to conflicts with other road users, leading to more conservative behavior than expected.

The back-propagation (BP) algorithm has emerged as a major breakthrough in the development of neural networks and is the basis of many deep learning training methods. It constantly iterates the parameters of the model to gradually narrow the difference between the estimated result and ground truth through the gradient descent method. Therefore, following this approach, the human driver preference estimation can be updated through Eq.~\ref{eq: bp}.
\begin{equation}
\begin{split}
\begin{aligned}
    &\alpha_{i}^{t} = \alpha_{i}^{t-1} - \mu \frac{\partial (\widehat{r}_{s,i}^{t-1} - r_{s,i}^{t-1})}{\partial \alpha_{i}^{t-1}} \\
    &\beta_{i}^{t} = \beta_{i}^{t-1} - \mu \frac{\partial (\sum\limits_{j\in N,j\neq i} \widehat{r}_{g,ij}^{t-1} - \sum\limits_{j\in N,j\neq i} r_{g,ij}^{t-1})}{\partial \beta_{i}^{t-1}}
    \label{eq: bp}
\end{aligned}
\end{split}
\end{equation}
where $\mu$ is the learning rate, $\widehat{r}_{s,i}^{t-1}$ and $\widehat{r}_{g,ij}^{t-1}$ are the rewards derived from potential game-estimated action $\widehat{a}_{i}^{t-1}$, and ${r}_{s,i}^{t-1}$ and ${r}_{g,ij}^{t-1}$ are the rewards derived from HDV-observed action ${a_{i}^{t-1}}$.

Through the comparison between rewards, $\alpha$ and $\beta$ gradually fit the true degree of selfishness and altruism of every HDV in our cooperative driving system. As a result, the output of the APG for HDVs will be closer to the observed action, with the risks associated with misunderstanding and misjudgment being minimized. Finally, Eq.~\ref{eq: potential function} can then be written as follows:
\begin{equation}
    \mathcal{S}(\boldsymbol{a^*}) = \sum_{t=0}^{T}\gamma^{t}(\sum\limits_{i \in N}\varphi_i \cdot\alpha_{i}\cdot r_{s,i}^{t}(a_{i}) +  \sum\limits_{i \in N}\sum\limits_{j\in N,j< i}\frac{\varphi_i+\varphi_j}{2} \cdot\frac{\beta_i+\beta_j}{2}\cdot r_{g,ij}^{t}(a_{i},a_{j}))
    \label{eq: finally potential function}
\end{equation}
By optimizing Eq.~\ref{eq: finally potential function} and updating its parameters, the optimal cooperative action list for mixed-traffic scenarios can be obtained.

Algorithm~\ref{alg: cooperative driving modeling process} outlines the cooperative decision-making process of the APG framework in mixed-traffic environments. It takes the vehicle type, previous position, and individual driving preferences as input. The framework first generates the optimal cooperative action while simultaneously updating its estimation of HDV preferences on the basis of their observed decisions. 

\begin{algorithm}[htbp]
\caption{Cooperative driving process in mixed traffic.}\label{alg: cooperative driving modeling process}
\KwIn{Vehicle state list $\boldsymbol{s}$, Vehicle type list $\boldsymbol{\mathcal{T}}$, Preference of each vehicle at last frame $\boldsymbol{\alpha^{t-1}}$, $\boldsymbol{\beta^{t-1}}$\, Shapley value each vehicle at last frame $\boldsymbol{\varphi_{t-1}}$;}
\KwOut{Vehicle action list $\boldsymbol{a^*}$, Vehicle trajectory list $\boldsymbol{\zeta}$}

Initialize action list $\boldsymbol{a^{t}} \leftarrow []$\;
Generate optimal cooperative action list $a^{APG}$ through Eq.~\ref{eq: finally potential function}\;
\ForEach{$i \in N$}{
    \uIf{$\mathcal{T}_i$ is not CAV}{Observe its observed action $a_{i}^{HDV}$\;
    Update the estimated preference $\alpha_{i}^{t}$ and $\beta_{i}^{t}$ of HDV $i$ through Eq.~\ref{eq: bp}\;
    Add action to list $\boldsymbol{a^*} \stackrel{+}\leftarrow a_{i}^{HDV}$\;
    }
    \uElse{Add action to list $\boldsymbol{a^*} \stackrel{+}\leftarrow a_{i}^{APG}$\;}
    Calculate the Shapley value $\varphi_{i}^{t}$ of each vehicle $i$\;
}
Update vehicle state $\boldsymbol{s^{t+1}}$ with action $\boldsymbol{a^*}$\;
Add state to trajectory $\zeta \stackrel{+}\leftarrow \boldsymbol{s^{t+1}}$\;
\end{algorithm}

\subsection{Generate the optimal action through searching the NE} 
Considering the large number of vehicles and the complex relationships in bottleneck areas, the problem to be solved involves many parameters and constraints. Additionally, since Eq.~\ref{eq: finally potential function} is a nonlinear equation, achieving real-time NE solving and parameter updates is challenging under these conditions. Therefore, we first perform a Taylor expansion of the system utility function and constraints and obtain the following:
\begin{equation}
\begin{split} \label{eq: problem after Taylor}
&\min \quad \mathcal{S}(x) = \frac{1}{2}\Delta x^{T} H_k\Delta x + \nabla \mathcal{S}(x_k)^{T}\Delta x \\
& s.t. \quad \left\{\begin{array}{lc}
\nabla c_{u}(x)^{T} \Delta x + c_{u}(x_k)\leq 0 \quad (u=1,2,...,p)\\
\nabla d_{i}(x)^{T} \Delta x + d_{i}(x_k)\leq 0 \quad (i=1,2,...,n)\\
\end{array}\right.
\end{split}
\end{equation}
where \( H_k \) is the second derivative of the system utility function at the \( k \)-th iteration, \( \nabla f(\boldsymbol{x_k}) \) is the first derivative of the system utility function, \( c_u \) represents the \( u \)-th constraint in the total \( p \) collision constraints, and \( d_i \) represents the acceleration constraint for vehicle \( i \) on the basis of its dynamics. Afterward, by introducing Taylor expansion at iteration point $\boldsymbol{x_k}$, the nonlinear objective function is reduced to a quadratic function, and its constraint conditions are reduced to linear functions. 

In the context of the constrained optimization problem as delineated in Eq.~\ref{eq: problem after Taylor}, the associated Lagrangian can be articulated as follows:
\begin{equation}
\begin{split}
    L(x, \lambda, \mu) &=\mathcal{S}(x)+\lambda\cdot\sum_{u=0}^{p}c_{u}(x_k)^{2} +\mu\cdot\sum_{i=0}^{n}d_{i}(x_k)^{2}\\
    &= \frac{1}{2}\Delta x^{T} H_k\Delta x + \nabla \mathcal{S}(x_k)^{T}\Delta x+\lambda\cdot\sum_{u=0}^{p}c_{u}(x_k)^{2} +\mu\cdot\sum_{i=0}^{n}d_{i}(x_k)^{2}
    \label{eq: lagrangian}
\end{split}
\end{equation}
Unlike typical sequential quadratic programming (SQP), we replace the constraints themselves with the residuals of the collision constraints and dynamics constraints in this case. During the optimization process, the optimal solution that satisfies the constraints is gradually approached using least squares rather than directly incorporating the constraints into the objective function. This approach effectively improves the optimization speed when there are many constraints.

By optimizing Eq.~\ref{eq: lagrangian}, the update direction $\Delta x$ can be obtained. The maximum step size $l$ that satisfies Eq.~\ref{eq: update direction} is determined through a backtracking line search, i.e.,
\begin{equation}
\begin{split}
    l = \text{argmax}\{l|\mathcal{S}(x_k+l\cdot\Delta x)\le \mathcal{S}(x_k)+ l\cdot\nabla \mathcal{S}(x_k)^{T}\Delta x\}
    \label{eq: update direction}
\end{split}
\end{equation}

Therefore, the iteration point \(x_{k+1} \) can be updated accordingly $x_{k+1} = x_k + \alpha_k \cdot \Delta x$. Furthermore, the Hessian matrix can be updated using the Broyden--Fletcher--Goldfarb--Shanno (BFGS) method. By iteratively updating an approximation of the Hessian matrix, computational costs are reduced.
\begin{equation}
    H_{k+1} = H_k + \frac{y_k y_k^T}{y_k^T x_k} - \frac{H_k x_k x_k^T H_k}{x_k^T H_k x_k}
    \label{eq: update Hessian}
\end{equation}
where $y_k=\nabla \mathcal{S}(x_{k+1})-\nabla \mathcal{S}(x_k)$ is the change in the gradient of the system utility. Finally, when the gradient condition satisfies the convergence criteria or the maximum number of iterations is reached, the approximate optimal solution \(x^* \) to the original function is returned:
\begin{equation}
    \left\| \frac{1}{2}\Delta x^{T} H_k\Delta x + \nabla \mathcal{S}(x^*)^{T}\Delta x+\lambda\cdot\sum_{u=0}^{p}c_{u}(x^*)^{2} +\mu\cdot\sum_{i=0}^{n}d_{i}(x^*)^{2} \right\|_\infty \leq \epsilon,
    \label{eq: convergence}
\end{equation}

In summary, by modeling the coordination problem as a potential game, we avoid explicitly assuming that players are inherently cooperative or competitive. Instead, competitive and cooperative behaviors emerge naturally through the optimization of the potential function. Algorithm~\ref{alg: p-func optimize} summarizes the main process of potential function optimization.

\begin{algorithm}[htbp]
\caption{Potential function optimization.}\label{alg: p-func optimize}
\KwIn{Initial action $a_0$, maximum iterations $I_{\text{max}}$, tolerance $\epsilon$}
\KwOut{Optimal cooperative action $\boldsymbol{a^*}$}
\While{$i < I_{\text{max}}$}
    {Construct the approximate objective function through Eq.~\ref{eq: problem after Taylor}\;
    Generate the Lagrangian function and determine the optimization direction with Eq.~\ref{eq: lagrangian}\;
    Determine the step size through backtracking line search through Eq.~\ref{eq: update direction}\;
    Update the Hessian matrix based on Eq.~\ref{eq: update Hessian}\;
    \If{Eq.~\ref{eq: convergence} is satisfied}{
    break
    }
}
Output the optimal cooperative solution $\boldsymbol{a^*}$\;
\end{algorithm}

\section{Experiment and Analysis}
To validate the effectiveness of the proposed APG framework, we first conduct an ablation study to compare success rates with and without adaptive updates with various penetration rates. We then benchmark its performance against that of other representative cooperative methods to demonstrate its advantages in terms of safety and efficiency. Finally, we evaluate the framework through real-world vehicle tests in various scenarios at Tongji Small Town (TJST).

Accidents at intersections account for 44\% of all traffic accidents, with more than 2.8 million incidents reported annually \citep{azimi2014stip}, making intersections among the most challenging environments for autonomous driving. Furthermore, according to the Beijing Autonomous Vehicle Road Test Report, over 39\% of CAV disengagement events occur at intersections, highlighting the difficulties associated with operating in mixed-traffic environments \citep{beijing2022}. Therefore, we conduct both simulations and field tests using unsignalized intersections as a representative scenario. Additional videos in roundabouts and other environments are available at \url{https://fangshiyuu.github.io/AWSW-PG/}.

\subsection{Experimental Setup}
\subsubsection{Reproduction of HDVs} A realistic background traffic simulation is essential for evaluating a cooperative driving system. While the intelligent driver model (IDM) is typically used, it cannot capture reciprocal decision-making, limiting the  interactive behavior of HDVs. To address this issue, we model HDVs with various preferences by integrating IRL with a noncooperative Bayesian game.

On the basis of our previous research \citep{10529605}, HDV decisions are reproduced through a four-step process. First, interaction data were collected at an intersection in Shanghai. Second, drivers' decisions were classified into three groups using K-means and the elbow method \citep{cui2020introduction, umargono2019k}. Third, a noncooperative Bayesian game modeled HDV decisions, using six typical actions \citep{tian2020game}. Finally, IRL was used to calibrate the HDVs' intrinsic preferences, and drivers were classified as aggressive, conservative, or normal. 

\subsubsection{Implementation Details} 
At the beginning of each case, we generated the initial position and initial speed of each vehicle. The initial position was randomly selected within a range of 30 m to 50 m from the entrance lane to its corresponding stop line, and the initial velocity was randomly selected within a range of 6 m/s to 10 m/s. Finally, according to the preset penetration rates, the vehicle types were randomized. The initial reward weights in Eq.~\ref{eq: finally potential function} were chosen to be $\alpha=[2,1,0.05]$ and $\beta=10$. The discount factor, planning horizon, learning rate, and radius of the cooperation system were set to $\gamma=0.9$, $T=8$, $\mu=1$, and \(R=\) 80 m, respectively. All these parameter choices were validated through pre-experiments. For each penetration rate, 100 trials involving approximately 800 vehicles and 40,000 interaction events in total were conducted. Furthermore, all the experimental results were statistically validated, demonstrating that the performance of the proposed APG framework significantly differs from that of all the other methods. 

\subsection{Ablation Studies on APG}

First, to evaluate the effectiveness of the proposed adaptive weight method, we conducted an ablation study on the APG framework. Specifically, two key components were independently disabled for comparison: (1) the Shapley value update, which governs the adaptive adjustment of the system-based impact of each vehicle, and (2) the back-propagation module, which updates HDV preference estimations on the basis of observed behaviors.
Table~\ref{tab: ablation studies} presents the success rates at different levels of CAV penetration, where the numbers in parentheses denote the performance differences compared with the full APG configuration.
In addition, a scenario is considered successful if it avoids both collisions and inefficiencies, where inefficiency is defined as the failure of all vehicles to clear the intersection within a 40-second time limit.


\begin{table*}[ht]
    \caption{Ablation studies on an APG with various rates of penetration }
    \centering
    \begin{tabular}{c c| c |c |c |c |c |c}
    \hline
     \multicolumn{2}{c|}{Component} & \multicolumn{6}{c}{Penetration rates}  \\
     \hline
      Shapley & BP & 37.5\% & 50\% & 62.5\% & 75\% & 87.5\% & 100\%\\
      \hline
       $\times$ & $\times$ & 79\%{\footnotesize{(-14\%)}} & 87\%{\footnotesize{(-6\%)}} & 88\%{\footnotesize{(-9\%)}} & 90\%{\footnotesize{(-7\%)}} & 92\%{\footnotesize{(-7\%)}} & 96\%{\footnotesize{(-3\%)}} \\ 
       $\times$ & $\checkmark$ & 89\%{\footnotesize{(-4\%)}} & 90\%{\footnotesize{(-3\%)}} & 86\%{\footnotesize{(-11\%)}} & 91\%{\footnotesize{(-6\%)}} & 95\%{\footnotesize{(-4\%)}} & 96\%{\footnotesize{(-3\%)}} \\ 
      $\checkmark$  & $\times$ & 90\%{\footnotesize{(-3\%)}} & 91\%{\footnotesize{(-2\%)}} & 90\%{\footnotesize{(-4\%)}} & 93\%{\footnotesize{(-4\%)}} & 94\%{\footnotesize{(-5\%)}} & 97\%{\footnotesize{(-2\%)}} \\ 
      $\checkmark$ & $\checkmark$ & \cellcolor{gray!20}93\% & \cellcolor{gray!20}93\% & \cellcolor{gray!20}97\% & \cellcolor{gray!20}97\% & \cellcolor{gray!20}99\% & \cellcolor{gray!20}99\% \\ 
      \hline
\end{tabular}
    \label{tab: ablation studies}
\end{table*}

As shown in Table~\ref{tab: ablation studies}, both the adaptive update of the Shapley values and the refinement of the HDV estimations significantly improve the rate of success across all levels of penetration. Specifically, without the adaptive weight method, the average rate of success decreases by 7.5\%. When only the Shapley value update or the HDV estimation refinement is applied, the average rate of success decreases by 4\% or 5.2\%, respectively. These results demonstrate that the proposed adaptive weight method plays a crucial role in enhancing the applicability of APGs in mixed traffic environments.


To further illustrate the effect of the adaptive weight method, we visualize a scenario with identical initial conditions. The vehicle trajectories without adaptive weight updates are shown in Fig.~\ref{fig: pg-ablation-without}. In this scenario, there are four CAVs, two aggressive HDVs, one normal HDV, and one conservative HDV. The number displayed in the upper right corner of each vehicle indicates its current speed.  

\begin{figure*}[ht]
  \begin{center}
  \centerline{\includegraphics[width=6.5in]{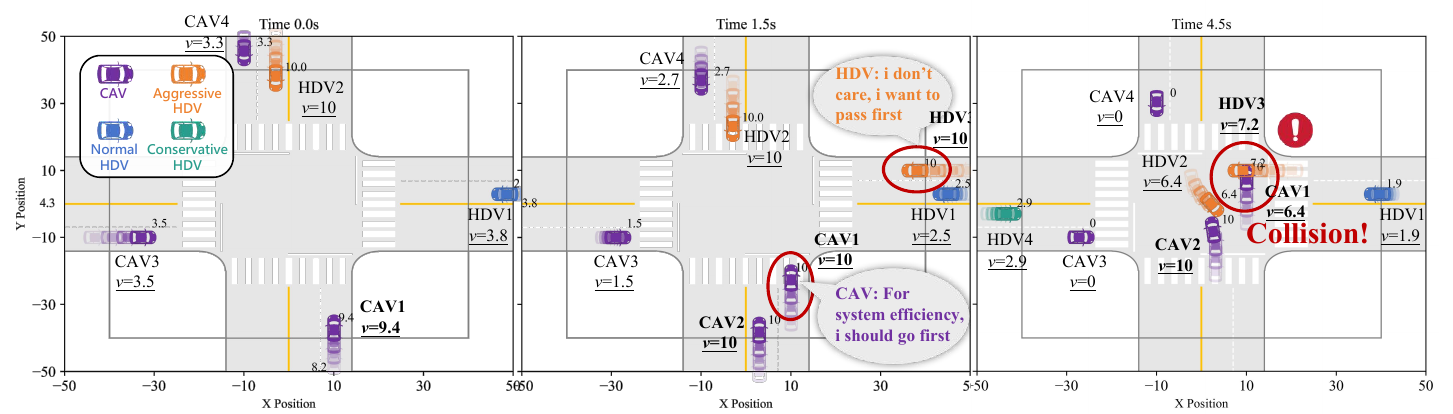}}
  \caption{Vehicle trajectories without adaptive weight updates.}\label{fig: pg-ablation-without}
  \end{center}
  \vspace{-0.8cm}
\end{figure*}

In this case, CAV1 had already crossed the stop line and entered the intersection when HDV3 arrived. To optimize the overall efficiency, the system prioritized CAV1's passage, allowing it to maintain a desired speed of 10 m/s. However, HDV3, driven by its preference to maximize personal efficiency, continued accelerating after entering the cooperative system, disregarding the broader impact of the system and ultimately causing a collision.

Furthermore, Fig.~\ref{fig: pg-ablation-with} illustrates the results after the adaptive weight method was applied. In this case, CAV1 and CAV2 detection revealed that HDV3 did not decelerate as initially estimated by the APG but instead maintained a relatively high speed. According to Eq.~\ref{eq: bp}, the system updated its estimation, recognizing that HDV3 had a stronger preference for efficiency. Consequently, during the optimization process, priority was given to maximizing the speed of HDV3, prompting CAV1 and CAV2, which potentially conflicted with HDV3, to slow. 

\begin{figure*}[t]
  \begin{center}
  \centerline{\includegraphics[width=6.5in]{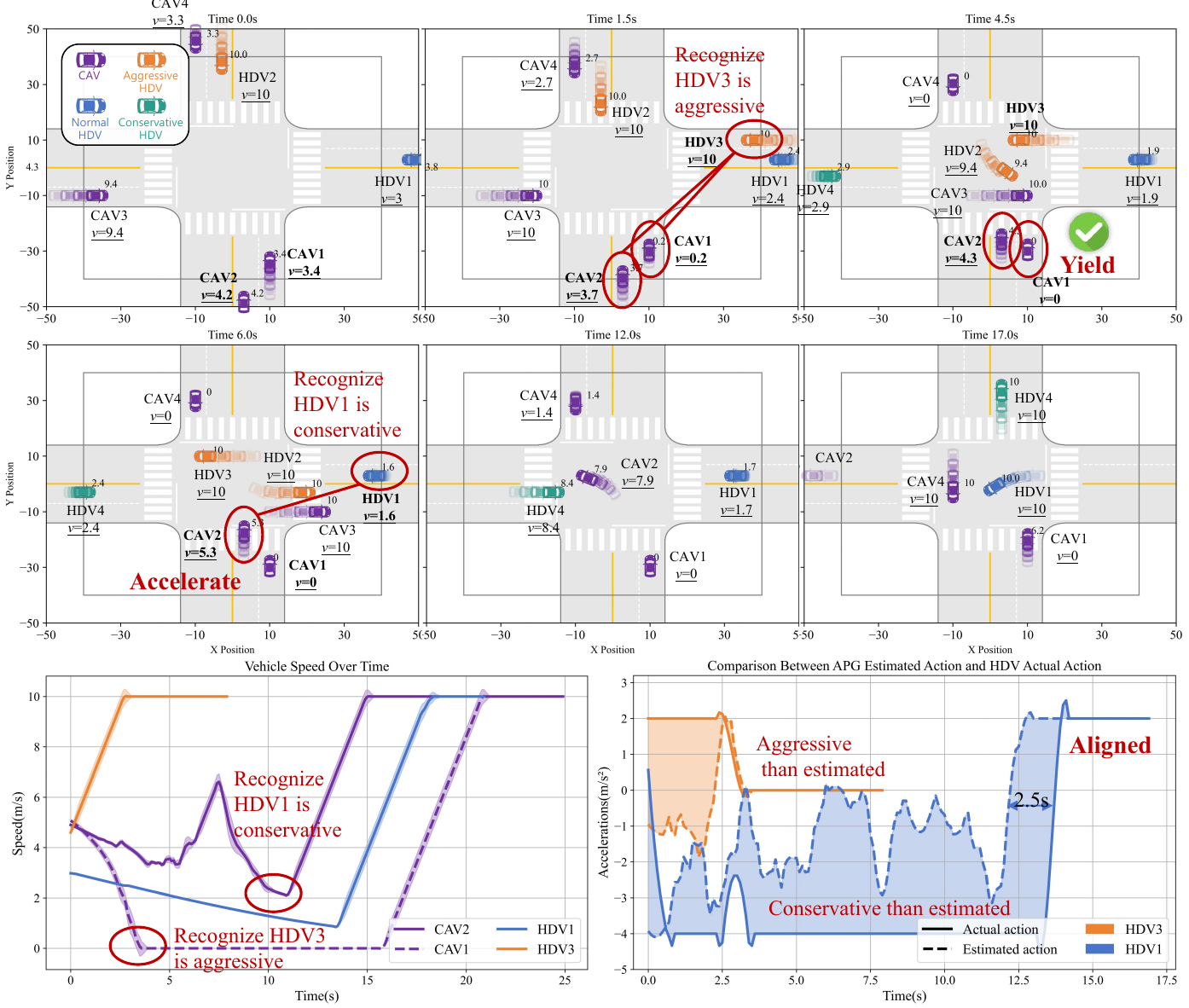}}
  \caption{Vehicle trajectories, speeds, and accelerations with adaptive weight updates.}\label{fig: pg-ablation-with}
  \end{center}
  \vspace{-0.8cm}
\end{figure*}

Moreover, during the interaction, HDV1 consistently exhibited more conservative behavior than initially estimated, leading the APG to prioritize CAV2's acceleration and to clear the intersection efficiently. A notable pattern emerges when estimated and observed actions are compared: at 11(13 s and 13(15 s, APG's predicted trajectories for HDV1 remain nearly identical. This occurs because, from a system perspective, when vehicle behaviors are controllable, earlier acceleration improves overall efficiency. In contrast, from an individual perspective, when other vehicles' actions are uncertain, HDV1 tends to delay acceleration until absolute safety is ensured. These findings reinforce that when HDVs adhere to the APG's optimal cooperative solutions, both system efficiency and individual performance improve. This will be further validated in future field tests.

In summary, with the adaptive weight method, all the vehicles successfully and efficiently navigated through the intersection without collisions.

\subsection{Comparison with Different Methods}
To further validate the effectiveness of the proposed APG in mixed-traffic cooperative scenarios, we compared it with several cooperative decision-making methods, including Projected-IDM (PIDM) \citep{albeaik2022limitations}, iDFST \citep{chen2022conflict}, CGame \citep{hang2021cooperative}, the Monte Carlo tree search (MCTS) \citep{kurzer2022learning}, Multiagent path finding (MAPF) \citep{yan2024multi}, and temporal safety constraints (TSCs) \citep{arfvidsson2024towards}.

\subsubsection{Overall Performance}
Fig.~\ref{fig: success_rate} illustrates the rates of success of the different methods, revealing two distinct trends. First, the rates of success of the PIDM, iDFST, MCTS, MAPF and TSC approaches decline as the level of penetration increases. A deeper analysis suggests that this is due primarily to a decrease in average efficiency at higher levels of penetration, which will be discussed in detail.

\begin{figure*}[ht]
  \begin{center}
  \centerline{\includegraphics[width=5in]{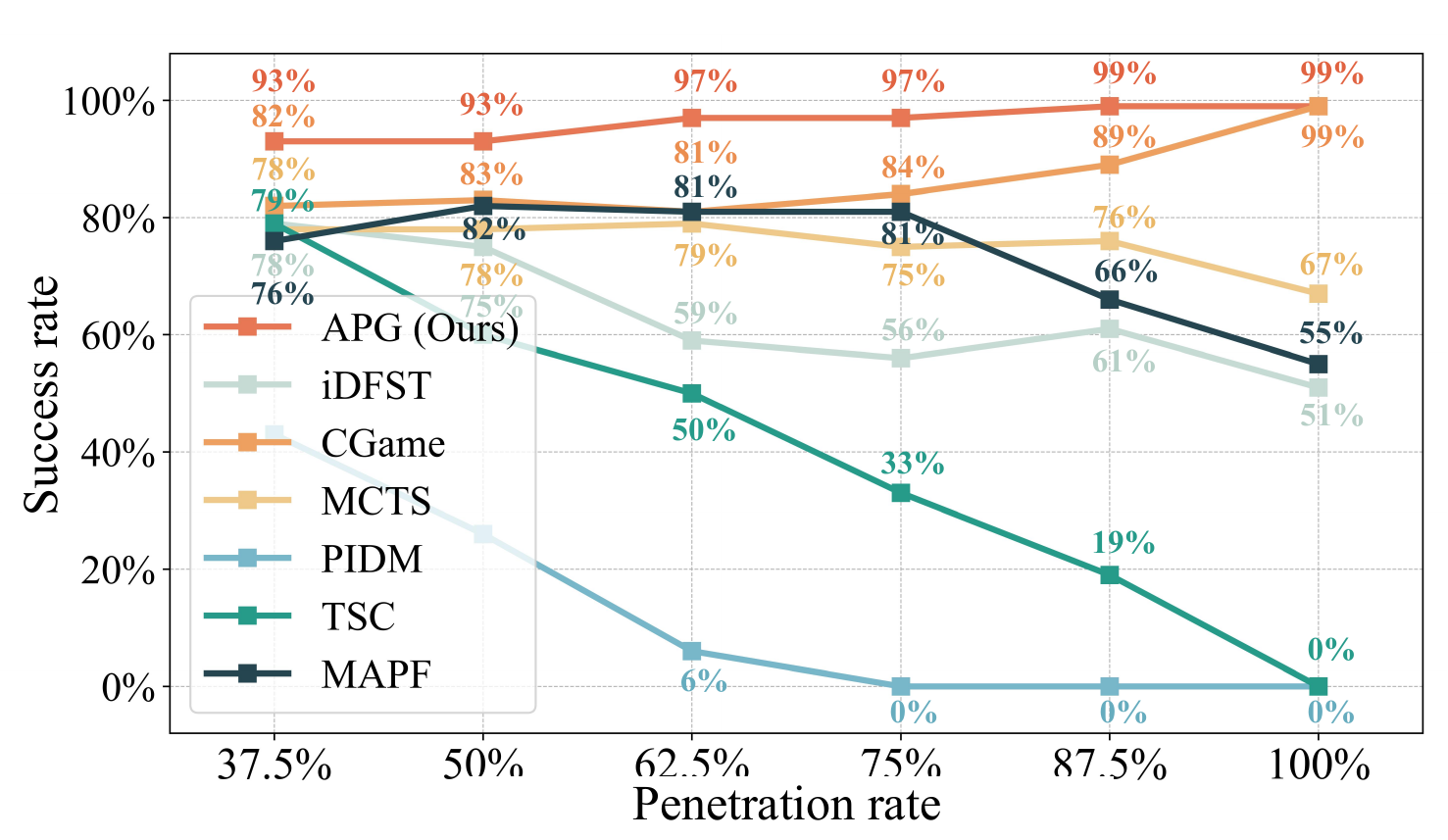}}
  \caption{Rates of success of different methods with various rates of penetration.}\label{fig: success_rate}
  \end{center}
  \vspace{-0.8cm}
\end{figure*}

Here, we emphasize a counterintuitive observation: while a higher proportion of CAVs should improve overall system performance, both real-world open-road data and our simulation results indicate the opposite. Instead of increasing efficiency, current CAVs are often the main contributors to its decline, which aligns with existing challenges in real-world deployments.

In contrast, CGame and APG exhibit completely different trends, with rates of success steadily increasing as the rate of penetration increases. This is mainly because game-theoretic methods are based on rational decision-making logic, making them more adaptable in mixed traffic environments. However, the rate of success of CGame remains lower than that of an APG because of its overly strong assumption that all cooperative participants will act collaboratively, indicating that various preferences of real-world human drivers are neglected. We further explore this issue in the Safety Evaluation section.

\subsubsection{Safety Evaluation}

The rate of success reflects the overall performance of the different methods. To further evaluate the proposed APG framework in terms of safety and efficiency, we introduce the collision rate and delay as key metrics. Table~\ref{tab: collision rate} presents the performance of the various methods with different rates of penetration.  

According to Table~\ref{tab: collision rate}, two key observations emerge from the results. First, the APG cooperative decision-making framework consistently achieves the lowest collision rate across all levels of penetration. Notably, when the rate of penetration exceeds 50\%, the collision rate decreases to zero. Second, at 100\% penetration, the collision rate is zero for all methods. These results can be attributed to the absence of uncontrolled HDVs in the environment. Since all the vehicles are CAVs following the same decision logic, ROW conflicts naturally do not occur.

Additionally, several other findings can be drawn from Table~\ref{tab: collision rate}. In terms of both the average and maximum rates of collision, CGame performs the worst, even falling behind rule-based methods such as iDFST. This further validates our earlier assertion that most existing game-theoretic methods assume that participants are either fully cooperative or entirely competitive. As a result, their performance deteriorates in the presence of HDVs with various preferences. Furthermore, the rationale behind the proposed APG approach, which avoids making explicit assumptions about participants and instead seeks equilibrium between the individual and the system, is revealed. These results also indicate that the APG better aligns with HDV decision-making logic and is more suitable for mixed-traffic environments.

\begin{table*}[ht]
    \caption{Comparison of the rates of collision for methods with various rates of penetration }
    \centering
    \centering
    \begin{tabular}{c|c |c |c |c |c |c}
    \hline
     \multirow{2}{*}{Models} & \multicolumn{6}{c}{Rates of penetration}  \\
     \cline{2-7}
      & 37.5\% & 50\% & 62.5\% & 75\% & 87.5\% & 100\%\\
      \hline
      PIDM & 5\% & 4\% & 4\% & 2\% & 6\% & \cellcolor{gray!20}0\% \\ 
      iDFST & 5\% & 6\% & 8\% & 11\% & 1\% & \cellcolor{gray!20}0\% \\ 
      CGame & 9\% & 11\% & 18\% & 10\% & 9\% & \cellcolor{gray!20}0\% \\
      MCTS & 9\% & 9\% & 7\% & 4\% & 3\% & \cellcolor{gray!20}0\% \\
      TSC & 21\% & 15\% & 8\% & 4\% & 4\% & \cellcolor{gray!20}0\% \\
      MAPF & 23\% & 18\% & 19\% & 6\% & 5\% & \cellcolor{gray!20}0\% \\
      APG & \cellcolor{gray!20}1\% & \cellcolor{gray!20}1\% & \cellcolor{gray!20}0\% & \cellcolor{gray!20}0\% & \cellcolor{gray!20}0\% & \cellcolor{gray!20}0\% \\ 
     \hline
\end{tabular}
    \label{tab: collision rate}
\end{table*}


Furthermore, in real-world interactions, collisions reflect only the lower bound of safety performance. Situations involving extremely small intervehicle gaps during alternating passages are also considered highly dangerous. To provide a more comprehensive evaluation of safety, we introduce the postencroachment time (PET) metric \citep{peesapati2018can, cooper1984experience}, which measures the temporal gap between two vehicles that are successively leaving the shared conflict area. According to previous studies, when the PET is less than 1 s, the interaction can be classified as high risk and should be avoided whenever possible.

Fig.~\ref{fig: pg-PET} presents the PET distributions for the different methods. Overall, the proposed APG framework maintains relatively large PET values, indicating a high margin of safety during interactions. The numbers shown in the figure denote the proportion of high-risk events excluding collisions. Specifically, our APG method results in only 1.4\% of high-risk events, which is slightly higher than the 1.3\% associated with the PIDM. Considering that the PIDM explicitly constrains intervehicle following distances, achieving a comparable level of safety demonstrates the effectiveness of our method. Combined with the collision rate results in Table~\ref{tab: collision rate}, the APG framework remains the most safety-oriented approach overall.

\begin{figure*}[ht]
  \begin{center}
  \centerline{\includegraphics[width=4.5in]{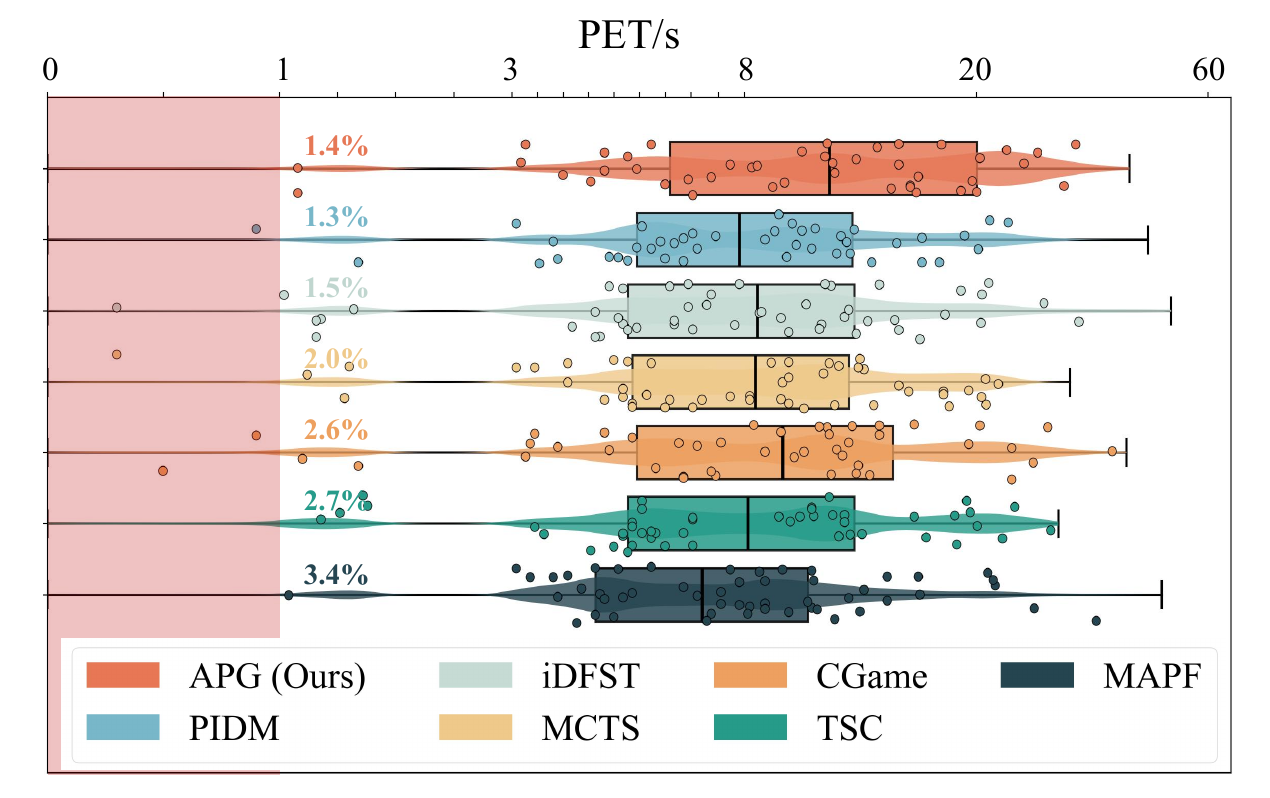}}
  \caption{PET distribution of different methods.}\label{fig: pg-PET}
  \end{center}
  \vspace{-0.8cm}
\end{figure*}

\subsubsection{Efficiency Evaluation}

Delay is a key indicator of driving efficiency. The average delay performance of different methods with various rates of penetration is shown in Fig.~\ref{fig: pg-delay}. The proposed APG framework consistently achieves the lowest delay at all levels of penetration. Notably, at 100\% penetration, the delay is only 7.6 s, with cooperative participants rarely needing to decelerate to a full stop. This finding demonstrates that incorporating individual utility into the modeling of system utility significantly increases both individual efficiency and system-wide efficiency. Additionally, in CGame, as the rate of penetration increases, more vehicles participate in cooperation, leading to a reduction in the average delay. 

In contrast, the PIDM, iDFST, MCTS, MAPF and TSC approaches exhibit entirely opposite trends in efficiency, where performance declines as the rate of penetration increases. At first glance, this seems contradictory to previous studies, but most prior research used the IDM as the background traffic model. This overly rigid design effectively forced HDVs to operate in a first-come, first-served manner, deviating significantly from real-world driving behavior. Consequently, when these methods are applied to environments with various HDVs, their performance trends differ drastically. Notably, at lower rates of penetration, the MAPF and TSC achieve slightly shorter delays than the proposed APG framework does. Nevertheless, their significantly higher collision rates and lower PET values suggest unsafe behaviors, making these methods unsuitable for mixed-traffic scenarios.

\begin{figure*}[ht]
  \begin{center}
  \centerline{\includegraphics[width=5in]{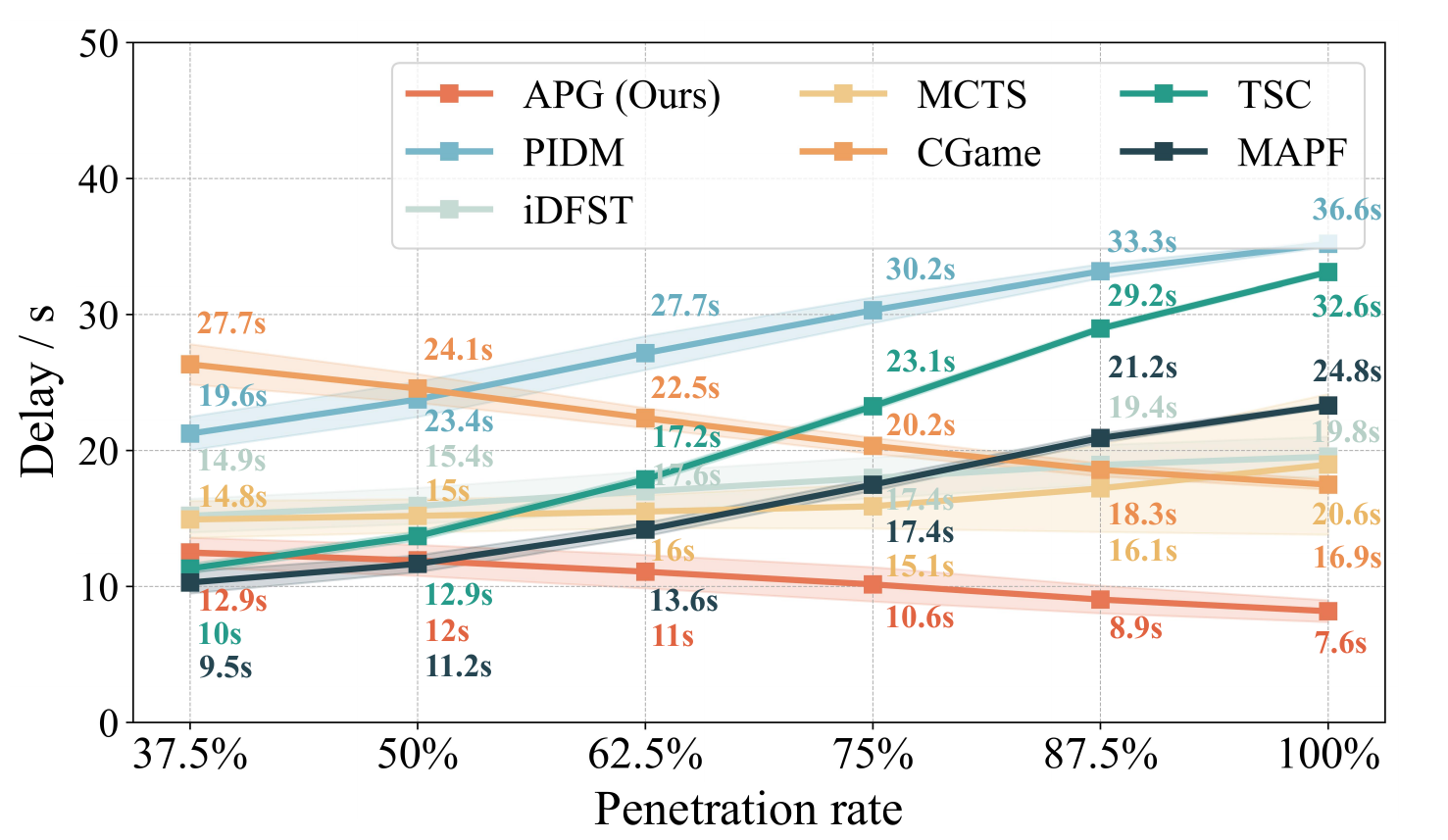}}
  \caption{Average delay performance of different methods with various rates of penetration.}\label{fig: pg-delay}
  \end{center}
  \vspace{-0.8cm}
\end{figure*}

In summary, a comprehensive comparison of APGs with other methods in terms of safety and efficiency reveals that APGs consistently outperform the other approaches across all rates of penetration. This superiority stems primarily from their design, which does not impose assumptions on participant behavior. Instead, APGs derive system utility step by step from a general formulation of individual utility. By leveraging the monotonic relationship between individual and system utility, they simultaneously achieve balanced optimization, making APGs particularly suitable for mixed-traffic environments.

\subsection{Real-World Field Test}
Finally, to verify the feasibility of applying the proposed method in real-world autonomous driving systems, we conducted experiments in multiple scenarios at TJST using a virtual-reality testing system. The main workflow of the testing system is illustrated in Fig.~\ref{fig: pg-vr}. By mapping APG-driven real-field CAVs and virtual HDVs into the same digital space, real-time interaction and cooperation were achieved.

During the experiments, a total of five to twelve vehicles, including two to three real CAVs and several virtual HDVs (V-HDVs), participated in different scenes, such as intersections and roundabouts. Each vehicle was equipped with onboard sensors, including LiDAR sensors, radar sensors, and cameras, and communication modules for V2X and Redis-based interaction to ensure synchronized information exchange between real and virtual agents. The proposed APG algorithm operates at a frequency of 10 Hz, generating control actions in real time on the basis of the latest perception and communication data. All the tests were conducted under the supervision of experienced safety drivers who could take over control at any time to ensure operational safety.

\begin{figure*}[ht]
  \begin{center}
  \centerline{\includegraphics[width=6.5in]{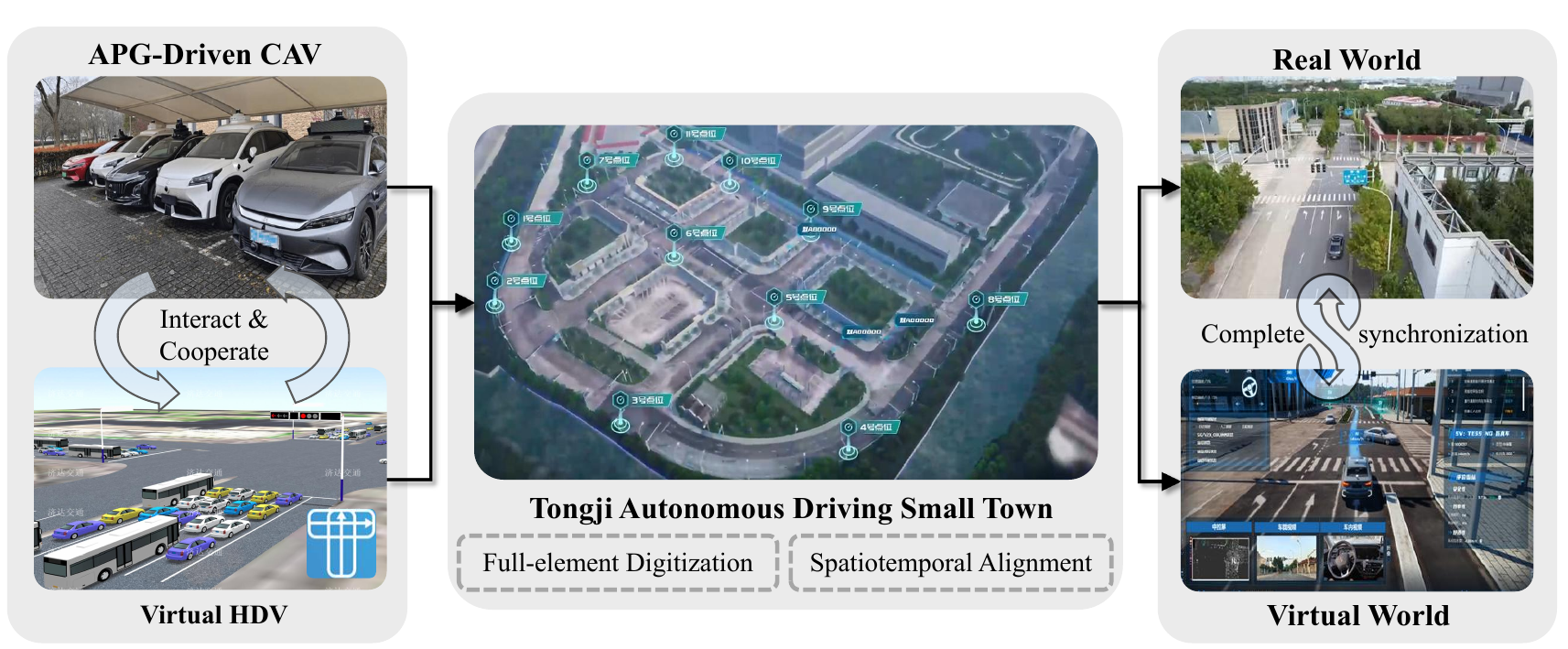}}
  \caption{Workflow of the virtual-reality testing system.}\label{fig: pg-vr}
  \end{center}
  \vspace{-0.8cm}
\end{figure*}

To demonstrate the ability of the proposed cooperative decision-making framework to enhance the performance of existing autonomous driving systems, we first designed a relatively simple scenario involving two CAVs and one HDV. In this case, the CAVs were controlled by the decision-making algorithm of a specific autonomous driving startup, while the HDV was operated by an experienced human driver. The initial positions and vehicle trajectories are shown in Fig.~\ref{fig: pg-field-deadlock}.

Additionally, we plotted speed versus time and distance to destination versus time for all the vehicles. The data show that all three vehicles continuously accelerated during the first 8 s. Around the 10 s mark, their close proximity led to abrupt braking, resulting in a deadlock. Shortly after, the HDV driver accelerated and exited the intersection first, while CAV1 and CAV2 remained stationary. Eventually, at approximately 24 s and 35 s, the safety drivers of CAV1 and CAV2 had to manually intervene to clear the intersection. This case highlights the challenges faced by current CAVs when dealing with multivehicle conflict scenarios in mixed traffic. 

\begin{figure*}[ht]
  \begin{center}
  \centerline{\includegraphics[width=6.5in]{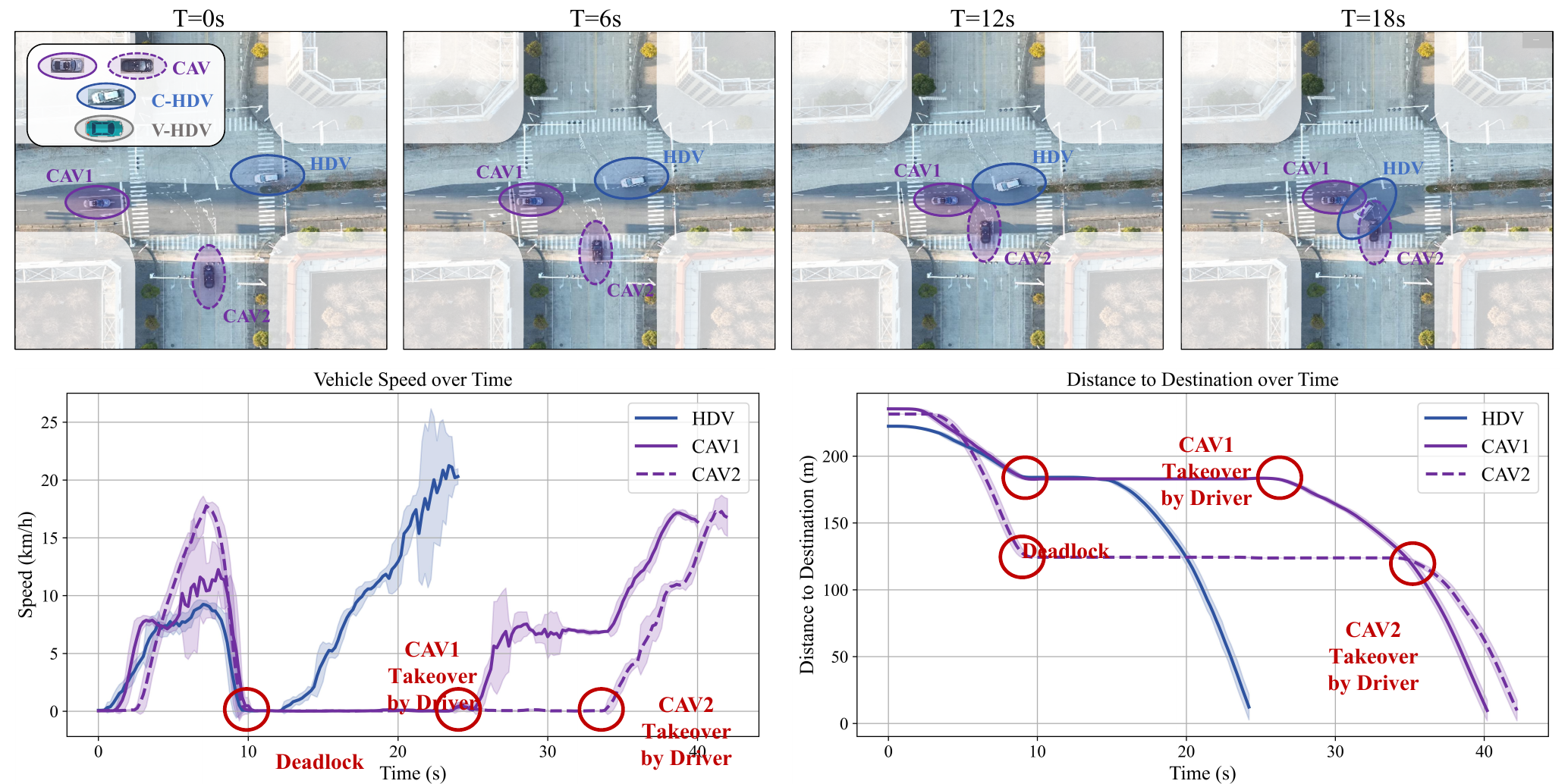}}
  \caption{Vehicle trajectories and performance analysis controlled by a commercial algorithm.}\label{fig: pg-field-deadlock}
  \end{center}
  \vspace{-0.8cm}
\end{figure*}

In contrast, the proposed APG cooperative decision-making method effectively addresses this challenge and remains robust in more complex scenarios. To further validate its effectiveness, we introduced three additional V-HDVs into the scenario shown in Fig.~\ref{fig: pg-vr-test}. Additionally, we equipped HDVs with onboard units (OBUs) to facilitate communication with CAVs and receive the optimal cooperative solutions generated by the APG in real time. However, human drivers were not required to follow the suggested actions of the APG.

During the first 3 s of the interaction, the system prioritizes overall efficiency by considering factors such as the lane position and distance to the intersection. As a result, it is optimal for CAV2, V-HDV2, and V-HDV1 to pass first, leading to a higher acceleration for CAV2, whereas CAV1 and the connected HDV (C-HDV) receive lower acceleration values. However, the uncontrolled V-HDV3 does not recognize this strategy and continues to accelerate. After its preference estimation is updated, the system determines that V-HDV3 is aggressive and intends to pass first. To prevent a collision, the strategy is adjusted by instructing CAV2 to decelerate until V-HDV3 clears the conflict point.  

Afterward, CAV2 sequentially passes the conflict points with CAV1 and C-HDV, with both vehicles accelerating in turn until they exit the intersection. During this process, the driver of C-HDV confidently considers accelerating at the 2 s mark. However, after just 1 s, the driver quickly notice several vehicles approaching the intersection at high speeds, prompting a reevaluation of their decision. As a result, the driver opts to follow the APG-recommended speed instead.

\begin{figure*}[ht]
  \begin{center}
  \centerline{\includegraphics[width=6.5in]{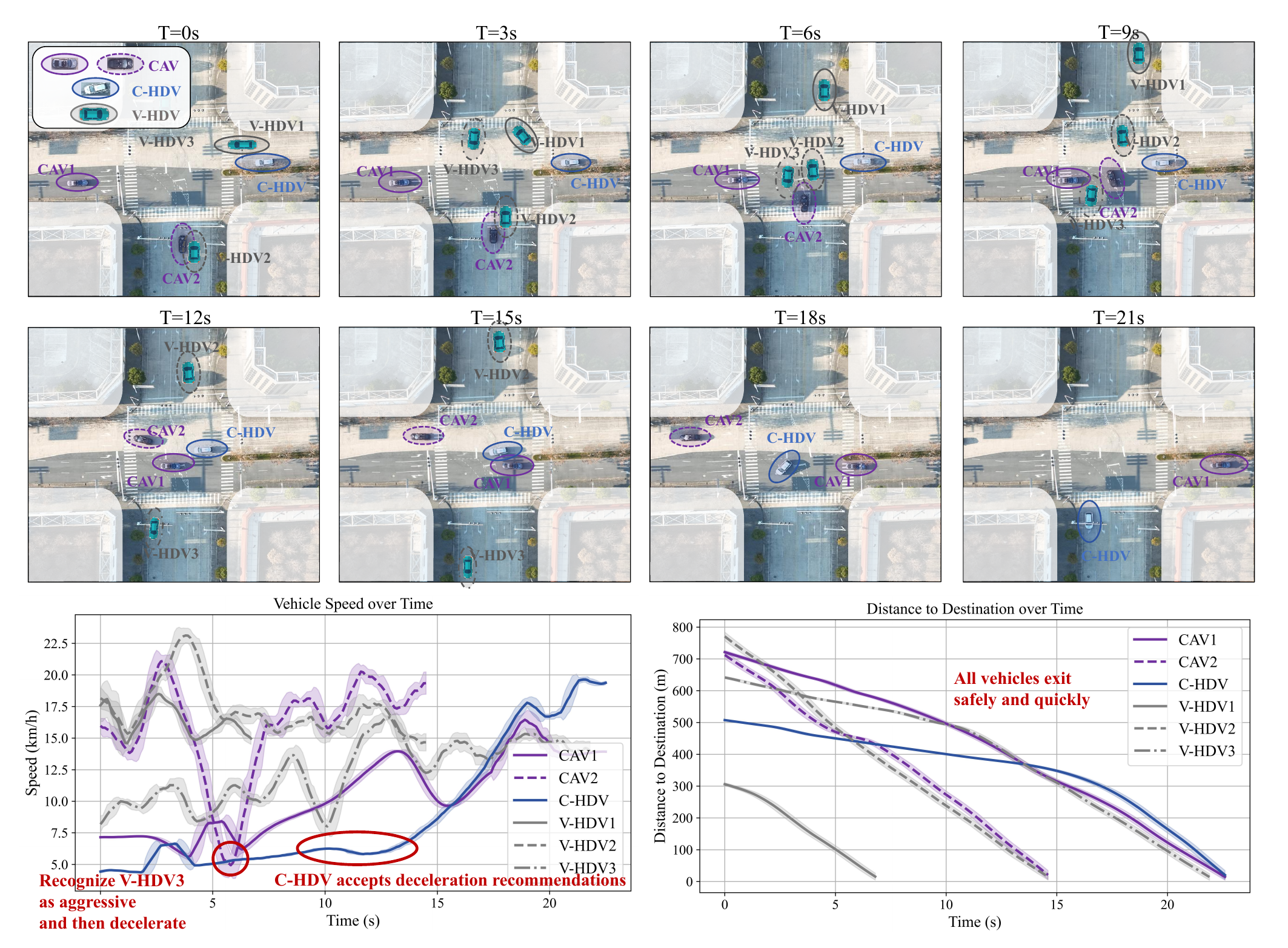}}
  \caption{Vehicle trajectories and performance analysis controlled by the proposed APG framework.}\label{fig: pg-vr-test}
  \end{center}
  \vspace{-0.8cm}
\end{figure*}

A comparison between Fig.~\ref{fig: pg-vr-test} and Fig.~\ref{fig: pg-field-deadlock} reveals that the proposed APG cooperative decision-making framework effectively handles various HDVs in mixed-traffic environments. By leveraging vehicle connectivity, it can transmit optimal cooperative actions to HDVs, guiding them from pursuing individual utility to achieving system-wide optimization and thereby demonstrating its applicability and superiority in real-world scenarios.

\section{Conclusion}
The advancement of autonomous driving and connectivity technologies has made cooperative driving feasible. While theoretically promising for improving traffic safety and efficiency, its performance in mixed-traffic environments has fallen short of expectations. This is due primarily to existing cooperative decision-making methods failing to account for individual needs and heterogeneity. 

To address this issue, we first construct a system utility function that is equipotential with changes in individual utility to ensure that individual equilibrium for all cooperative participants is simultaneously achieved when the system utility is optimized. This enables the joint optimization of both individual- and system-level benefits. To further address the heterogeneity among participants, an adaptive weight method is introduced. It dynamically computes Shapley values and refines the estimation of individual decision preference in real time, enhancing the adaptability of the model in mixed-traffic scenarios. Ablation experiments demonstrate that the proposed APG cooperative framework significantly improves the success rate of CAVs in mixed-traffic scenarios. Comparative evaluations against other cooperative methods further highlight the framework's advantages in terms of safety and efficiency. Finally, a field test was conducted by benchmarking the APG framework against the decision-making algorithm of an autonomous driving startup. The results show that the APG effectively prevents deadlocks and reduces the need for safety driver interventions, demonstrating its capability for real-world deployment.

Despite these promising results, deploying the APG framework in open-road environments still faces challenges related to data privacy and cybersecurity. Since the current cooperative communication relies on OBUs that exchange state and planning information, future work will focus on incorporating secure communication protocols and privacy-preserving mechanisms. In addition, collaboration with regulatory bodies will be necessary to ensure that cooperative driving systems comply with evolving standards and legal requirements for information security.


\section*{Acknowledgments}
This work was supported in part by the State Key Lab of Intelligent Transportation System under Project No. 2024-A002, the National Natural Science Foundation of China (52302502), the State Key Laboratory of Intelligent Green Vehicle and Mobility under Project No. KFZ2408, and the Fundamental Research Funds for the Central Universities.

\section*{Compliance with ethics guidelines}
Shiyu Fang, Xiaocong Zhao, Xuekai Liu, Peng Hang, Jianqiang Wang, Yunpeng Wang, and Jian Sun declare that they have no conflicts of interest or financial conflicts to disclose.

\bibliographystyle{references-format} 
\bibliography{references}

\newpage

\end{document}